%% file: main_neurips.tex
\documentclass{article}

\usepackage[preprint]{neurips_2026}

\usepackage[utf8]{inputenc} 
\usepackage[T1]{fontenc}    
\usepackage{hyperref}       
\usepackage{url}            
\usepackage{booktabs}       
\usepackage{amsfonts}       
\usepackage{nicefrac}       
\usepackage{microtype}      
\usepackage{xcolor}         

\usepackage{amsmath}
\usepackage{amssymb}
\usepackage{mathtools}
\usepackage{amsthm}
\usepackage{algorithm}
\usepackage{algpseudocode}
\usepackage{wrapfig}
\usepackage{subcaption}

\usepackage[capitalize,noabbrev]{cleveref}

\theoremstyle{plain}
\newtheorem{theorem}{Theorem}[section]

\theoremstyle{definition}
\newtheorem{definition}[theorem]{Definition}

\theoremstyle{remark}

\usepackage[disable]{todonotes}

\title{Randomness is sometimes necessary for coordination}

\author{%
  Rohan Patil$^*$ \quad\quad Jai Malegaonkar\thanks{Equal contribution.} \quad\quad Henrik I. Christensen \\
  Department of Computer Science and Engineering\\
  University of California San Diego\\
  San Diego, CA 92093 \\
  \texttt{\{rpatil, jmalegaonkar, hichristensen\}@ucsd.edu} \\
}

\begin{document}

\maketitle

\begin{abstract}
Full parameter sharing is standard in cooperative multi-agent reinforcement learning (MARL) for homogeneous agents. Under permutation-symmetric observations, however, a shared deterministic policy outputs identical action distributions for every agent, making role differentiation impossible. This failure can theoretically be resolved using symmetry breaking among anonymous identical processors, which requires randomness. We propose Diamond Attention, a cross-attention architecture in which each agent samples a scalar random number per timestep, inducing a transient rank ordering that masks lower-ranked peers from agent-to-agent attention while leaving task attention fully unmasked. This realizes a random-bit coordination protocol in a single broadcast round, and the set-based attention enables zero-shot deployment to teams of different sizes. We evaluate across three regimes that isolate when structured randomness matters. On the perfectly symmetric XOR game, our method achieves $1.0$ success while all deterministic baselines plateau near $0.5$. On control coordination tasks, a policy trained on $N=4$ generalizes zero-shot to $N \in [2,8]$. On SMACLite cross-scenario transfer, we achieve zero-shot transfer where standard baselines cannot transfer due to structural limitations. Furthermore, replacing the structured mask with standard dropout-based randomness results in a 0\% win rate, confirming that protocol-space structure, not stochastic noise, is the operative ingredient. \url{https://anonymous.4open.science/r/randomness-137A/}
\end{abstract}

\input{sections/introduction}

\input{sections/related}
\input{sections/motivation}
\input{sections/experiments}
\input{sections/conclusion}

\bibliography{ref}
\bibliographystyle{icml2026}

\clearpage
\appendix

\input{sections/appendix}


\end{document}

%% file: sections/introduction.tex
\section{Introduction}
\label{sec:introduction}

A widely adopted design choice in cooperative multi-agent reinforcement learning (MARL)
is full parameter sharing, where all homogeneous agents execute copies of a single learned
policy~\cite{gupta2017cooperative, terry2020revisiting}. Combined with the centralized
training and decentralized execution (CTDE) paradigm~\cite{lowe2017multi} and value
decomposition methods such as VDN~\cite{sunehag2017value} and
QMIX~\cite{rashid2020monotonic}, this approach reduces sample complexity, simplifies
training infrastructure, and scales naturally with team size. It has become the default
substrate for cooperative MARL~\cite{sunehag2017value, rashid2020monotonic,
gronauer2022multi, oroojlooy2023review}. This default fails in tasks with multi-modal
reward structures, where multiple equally optimal joint strategies exist. Under such
conditions, all agents produce identical action distributions when observations are
structurally identical, making role differentiation impossible. \citet{fu2022revisiting}
use the XOR game---where two agents must choose opposite actions to receive any
reward---to show that a shared deterministic policy outputs the same action for both
agents, guaranteeing zero reward regardless of training duration, illustrating that this
is not an empirical pathology that better optimization can resolve.

\citet{angluin1980local} proved that symmetry breaking among anonymous processors is
impossible deterministically; \citet{case2005learning} extended this to cooperative
games, establishing that access to shared random bits is \emph{necessary}, not merely
beneficial, for coordination when agents lack unique identifiers.
Section~\ref{sec:related} develops these theoretical foundations and surveys how prior
work falls short on each axis.

Existing approaches address this problem on two axes, but no prior method satisfies
both simultaneously. Methods that modify parameter sharing or optimization dynamics
improve coordination in symmetric settings but remain deterministic at execution time,
failing the theoretical randomness requirement. Methods that achieve symmetry breaking
through sequential execution resolve the symmetry problem but require
$\mathcal{O}(N)$ inference rounds and an externally imposed execution
order---which, in a truly decentralized homogeneous setting, is itself the
symmetry-breaking problem.

We operationalize the theoretical prescription of \citet{case2005learning} in a practical
MARL architecture. Our proposed \emph{Diamond Attention} is a cross-attention mechanism
between agent and task embeddings that incorporates structured random masking. At each
timestep, every agent samples a scalar random number and shares it with the team via a
single broadcast round. These scalars induce a strict rank ordering over agents, generating
asymmetric attention masks over the agent dimension: each agent masks all agents ranked
below it and attends only to agents at equal or higher rank. The mask applies only to
agent-to-agent attention; task attention remains fully unmasked. The resulting asymmetry
creates a dynamic, per-timestep hierarchy in which high-ranked agents attend to few peers
and act largely independently, while low-ranked agents condition their behavior on what
higher-ranked agents are doing. Because Diamond Attention operates over sets of agent and
task embeddings, the architecture accepts any number of agents and any number of tasks
without modification: policies trained on one team size generalize zero-shot to others as
an inherent property of computing attention over variable-length sequences.

Our contributions are as follows:
\begin{itemize}
  \item \textbf{Theory.} We formalize the equivalence between XOR coordination and random-bit sharing , providing the theoretical bridge that motivates each architectural component of Diamond Attention and grounds the structured mask in the necessity result of Case et al. (2005).
  \item \textbf{Architecture.} We propose Diamond Attention, which realizes this protocol
    via structured random masking in a single broadcast round while retaining the set-based
    scalability of standard cross-attention. The architecture requires no sequential
    execution and no external agent identifiers.
  \item \textbf{Empirical validation.} We validate across three regimes: Diamond Attention
    is the only approach to achieve $1.0$ success on the XOR game where all deterministic
    baselines plateau at the random-action floor; a policy trained on $N=4$ generalizes
    zero-shot to $N \in [2, 8]$ on VMAS continuous coordination tasks; and the architecture
    achieves zero-shot transfer from easier to harder SMACLite scenarios where all baselines
    and ablations fail entirely.
\end{itemize}

%% file: sections/related.tex
\section{Related Work}
\label{sec:related}

Parameter sharing (PS) has become the dominant paradigm in cooperative
MARL~\cite{gupta2017cooperative, terry2020revisiting}, especially when combined with
the CTDE framework~\cite{lowe2017multi} and value decomposition methods such as
VDN~\cite{sunehag2017value} and QMIX~\cite{rashid2020monotonic}. Under symmetric
observations a shared deterministic policy produces identical outputs for all agents, and
gradient updates that improve one agent's strategy identically affect all
others~\cite{fu2022revisiting}, trapping the system in symmetric equilibria. Recent work
mitigates this while preserving PS efficiency: Kaleidoscope~\cite{li2024kaleidoscope}
introduces learnable sparse masks to induce per-agent heterogeneity at training time,
GradPS~\cite{qingradps} resolves gradient conflicts that arise from opposing update
signals, and pH-MARL~\cite{sebastian2025physics} leverages port-Hamiltonian geometric
priors to enforce valid distributed coordination structures. However, Kaleidoscope's
learned masks are deterministic once trained; GradPS resolves gradient conflicts during
optimization but produces no execution-time differentiation between agents; and
pH-MARL's geometric priors enforce coordination structure without introducing the
stochasticity that theory demands. These methods address symptoms of the symmetry
problem rather than its theoretical root.

A related challenge is zero-shot scalability: deploying a trained policy to teams of
varying size without retraining~\cite{liu2024scaling}. UPDeT~\cite{hu2021updet}
achieves scalability by treating each agent's observation as part of a variable-length
sequence, but agents act simultaneously from independent observations, leaving symmetry
breaking unaddressed in multi-modal reward settings. Autoregressive models
MAT~\cite{wen2022multi} and Sable~\cite{mahjoub2024sable} achieve symmetry breaking by
construction through sequential execution, but at the cost of $\mathcal{O}(N)$ inference
latency and an externally imposed execution order---establishing that order requires a
coordination mechanism, and in a truly decentralized homogeneous setting, agreeing on
who executes first is equivalent to solving the symmetry-breaking problem, making the
approach circularly dependent on an assumption it cannot supply. We make no claim of
superior coordination quality where a fixed ordering is externally available; our
contribution is orthogonal, addressing the gap that UPDeT and MAT each leave open on a
different axis.

The theoretical underpinnings of our approach trace to distributed
computing. \citet{angluin1980local} proved that symmetry breaking in anonymous networks
is impossible deterministically. \citet{fischer1985impossibility} established that
consensus is unattainable in asynchronous systems even with a single fault, reinforcing
that deterministic coordination protocols are fragile even under mild adversarial
conditions. \citet{case2005learning} showed that shared random bits resolve this
impossibility for cooperative games, making coordination achievable with bounded failure
probability. The MP-MAB literature~\cite{liu2010distributed, shi2021multi} reaches the
same conclusion through collision models that are structurally identical to coordination
failures in multi-modal MARL: agents assigned to the same arm and agents selecting the
same action face the same orthogonalization problem, and in both cases deterministic
policies cannot escape.

Attention-based methods such as UPDeT achieve zero-shot scalability over variable team
sizes but leave symmetry breaking unaddressed in multi-modal reward settings.
Autoregressive methods such as MAT and Sable achieve symmetry breaking by construction
but require $\mathcal{O}(N)$ sequential rounds and an externally imposed execution
order. Diamond Attention addresses the intersection: structured randomness, grounded in
the theoretical necessity established by \citet{case2005learning}, enables coordination
in a single broadcast round while retaining the set-based scalability of standard
cross-attention. No prior MARL architecture has realized this protocol as an
architectural primitive.

%% file: sections/motivation.tex
\section{Motivation \& Model Architecture}
\label{sec:motivation}

We begin by formalizing the core limitation that motivates our approach.

\begin{definition}[Symmetry Breaking]
\label{def:symmetry_breaking}
In a cooperative multi-agent system with $n$ homogeneous agents sharing a single policy 
$\pi_\theta$, \emph{symmetry breaking} is the ability of agents to produce differentiated 
action distributions despite receiving structurally identical observations. Formally, if 
the observations $o_i$ and $o_j$ of agents $i$ and $j$ are equal under permutation of 
agent indices, then $\pi_\theta(o_i) = \pi_\theta(o_j)$ by parameter sharing, so 
differentiated behavior must arise from differentiated internal state rather than from 
the policy itself.
\end{definition}

We next detail the XOR game, which provides the theoretical grounding for our 
architecture, and derive how each component of the resulting protocol is realized in the 
Diamond Attention mechanism.

\subsection{XOR Game}
\label{subsec:xor_game}

The XOR game~\cite{fu2022revisiting} is a single-step cooperative game where two players 
each select from two actions; both receive reward $1$ if their actions differ, and $0$ 
otherwise (\cref{tab:xor_payoff}). Its generalization to $n$ players and $k$ actions 
($n \le k$) yields reward only when all players select distinct actions, mirroring 
collision models in the multi-agent multi-armed bandit literature~\cite{liu2010distributed}.

\begin{table}[h]
    \centering
    \begin{tabular}{|c|c|}
        \hline
        0 & 1 \\
        \hline
        1 & 0 \\
        \hline
    \end{tabular}
    \caption{Payoff matrix for the 2-player XOR game.}
    \label{tab:xor_payoff}
\end{table}

While an autoregressive approach can solve XOR by having agents act sequentially and 
condition on predecessors' actions~\cite{fu2022revisiting}, the required communication 
rounds grow linearly with team size. More fundamentally, in a truly decentralized 
homogeneous setting, establishing the execution order itself requires random bits, which 
is equivalent to solving the generalized XOR game.

\citet{case2005learning} demonstrate that coordination games like XOR can be solved in 
homogeneous settings by sharing random bits in a single communication round: each player 
generates a string of bits, and with bounded probability all strings are unique, enabling 
coordination through any fixed total ordering on bit strings. Without randomness, agents 
provably fail to obtain the optimal payoff in a truly decentralized setting. To formalize this, we model players as machines aligned with the Inexhaustible Interactive Turing Machine framework~\cite{kusters2013iitm}. In our formulation, machines do not engage in 
point-to-point communication but rather broadcast messages.\footnote{\citet{kusters2013iitm} have point-to-point communication} While the source of a 
broadcast is indistinguishable, agents can determine how many machines are broadcasting 
the same input --- realizable through a frequency-modulated receiver where signal 
strength indicates the number of broadcasting machines.


\begin{definition}[Player]
\label{def:player}
    A player is a machine initialized with a finite input that may execute four routines,
    transitioning between them based on the input: \textbf{Compute} (run any
    terminating Turing machine on the current tape, whose final state becomes the new
    input); \textbf{Broadcast} (transmit a portion of the input to other agents);
    \textbf{Receive} (append incoming broadcasts to the input); and \textbf{Sample}
    (append freshly sampled uniform random bits to the input). Each non-compute routine
    returns to Compute. The final input state on halt is the player's output.
\end{definition}
\begin{definition}[Homogeneity]
\label{def:homogeneity}
    Two players $A$ and $B$ are \emph{homogeneous} if, for any input $I$ and any string of 
    random bits $r$ supplied to both players' Sample routines, $A$ and $B$ produce the 
    same final output, both consume exactly the bits in $r$, and neither generates random 
    bits beyond $r$. 
\end{definition}

It is trivial to see that homogeneity is transitive.

\begin{theorem}
\label{thm:xor_execution_equivalence}
    The probability that $n$ homogeneous players solve the generalized XOR game ($n$ 
    players, $n$ actions) equals the probability of agreeing on an execution order when 
    each player generates and shares $k$ bits.
\end{theorem}

\begin{proof}
    \cref{app:proof}
\end{proof}

\subsection{From Theory to Model Architecture}
\label{subsec:bridge}

\cref{thm:xor_execution_equivalence} establishes that solving the generalized XOR game 
reduces to each player generating and sharing random bits to agree on an execution order. 
\cref{tab:theory_to_arch} maps each component of this theoretical protocol to its 
realization in Diamond Attention. The key insight is that attention masking serves as the mechanism by which random scalars 
induce differentiated behavior. Even though all agents share identical network parameters, 
different random numbers produce different effective attention patterns, yielding distinct 
latent representations and consequently distinct action distributions. This resolves the 
symmetry identified in \cref{def:symmetry_breaking} at the attention level rather than 
through parameter differences.

\begin{table}[h]
\centering
\small
\begin{tabular}{p{0.44\linewidth} p{0.48\linewidth}}
\toprule
\textbf{Theoretical Component} & \textbf{Architectural Realization} \\
\midrule
Player samples random bits 
(\cref{def:player}, Routine 4) &
    Agent samples $r_i \sim \mathrm{Uniform}[0,1]$, appended to its observation at every 
    timestep \\
\addlinespace
Players broadcast bits 
(\cref{def:player}, Routine 2) &
    All agents' random numbers are visible in the shared observation space \\
\addlinespace
Fixed strict ordering over bit strings &
    Pairwise comparison: if $r_j < r_i$, agent $i$ masks agent $j$ from its attention; 
    agents with equal or higher rank remain unmasked \\
\addlinespace
Execution order determines role assignment &
    Asymmetric mask creates a dynamic per-timestep hierarchy; high-ranked agents act as 
    temporary leaders attending primarily to task embeddings \\
\addlinespace
Collision probability 
(\cref{thm:xor_execution_equivalence}) &
    With $r_i \sim \mathrm{Uniform}[0,1]$, $P(r_i = r_j) = 0$; a strict total ordering 
    exists with probability $1$ \\
\bottomrule
\end{tabular}
\caption{\small Mapping from the theoretical coordination protocol of 
\cref{thm:xor_execution_equivalence} to the Diamond Attention architecture.}
\label{tab:theory_to_arch}
\end{table}

\label{subsec:diamond_attention}
\begin{figure}[h]
    \begin{minipage}[c]{0.6\linewidth}
        \centering
        \includegraphics[width=\linewidth]{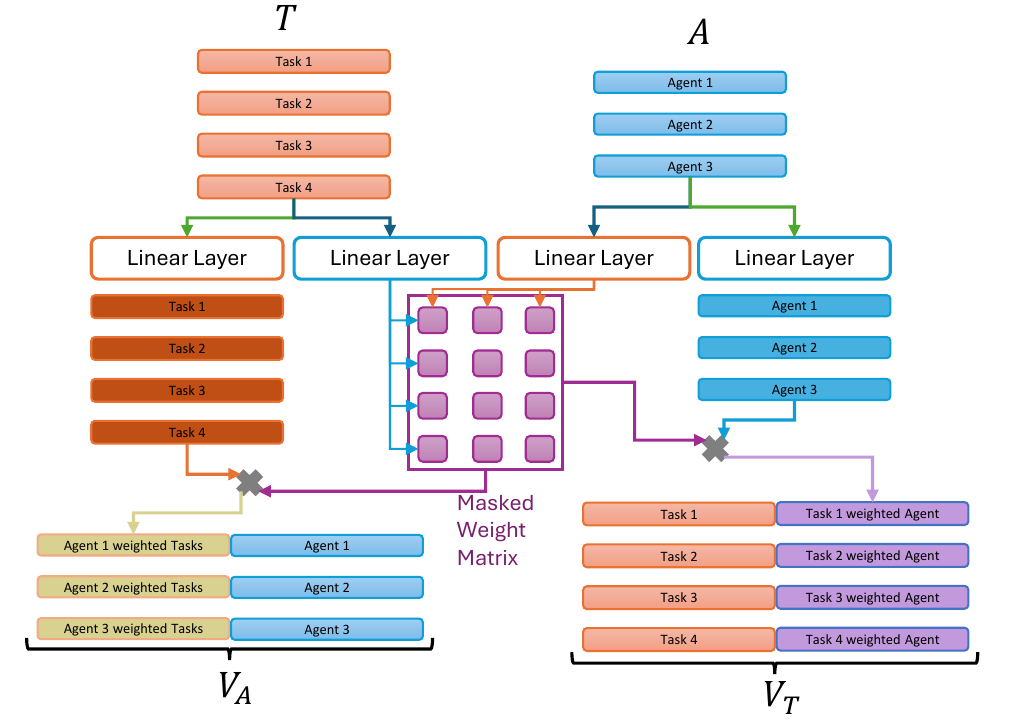}
        \caption{\small Calculation of $V_A$ and $V_T$ from agent embeddings $A$ and task 
        embeddings $T$. The weight matrix follows the standard attention mechanism; the 
        matrix multiplication uses the weight matrix resulting from masking the calculated 
        scores.}
        \label{fig:diamond_attention}   
    \end{minipage}
    \hfill
    \begin{minipage}[c]{0.37\linewidth}
        \centering
        \includegraphics[width=\linewidth]{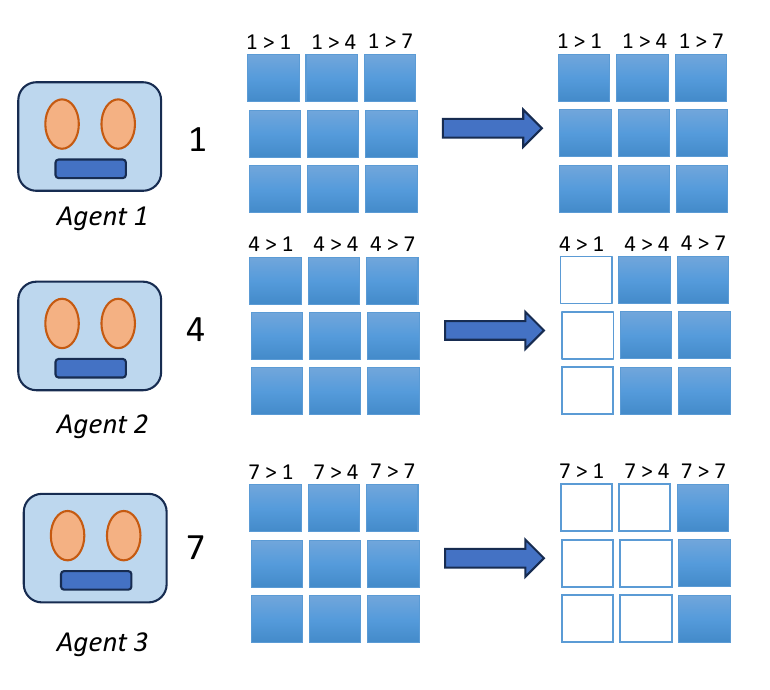}
        \caption{\small Mask generation: each agent generates a random number. For each column, the 
        system computes whether an agent's number exceeds others'; if so, the entry is 
        masked. Hollow squares represent masked entries in the task $\times$ agent mask 
        matrices.}
        \label{fig:masking}
    \end{minipage}
\end{figure}

We develop a cross-attention architecture that assigns distinct sub-goals to different 
agents. Our aim is an architecture that maps any number of agents to any number of tasks 
without modification. \cref{eq:attention_formula} shows the standard attention mechanism, where $Q, K, V$ are 
the query, key, and value matrices and $d_k$ is the key 
dimension~\cite{vaswani2017attention}. We employ cross-attention between agent embeddings and task embeddings, termed 
\emph{Diamond Attention} (\cref{fig:diamond_attention}).

\begin{equation}
\label{eq:attention_formula}
    \text{Attention}(Q, K, V) = \text{softmax}\!\left(\frac{QK^T}{\sqrt{d_k}}\right)V
\end{equation}

\begin{figure}[h]
    \centering
    \includegraphics[width=0.99\linewidth]{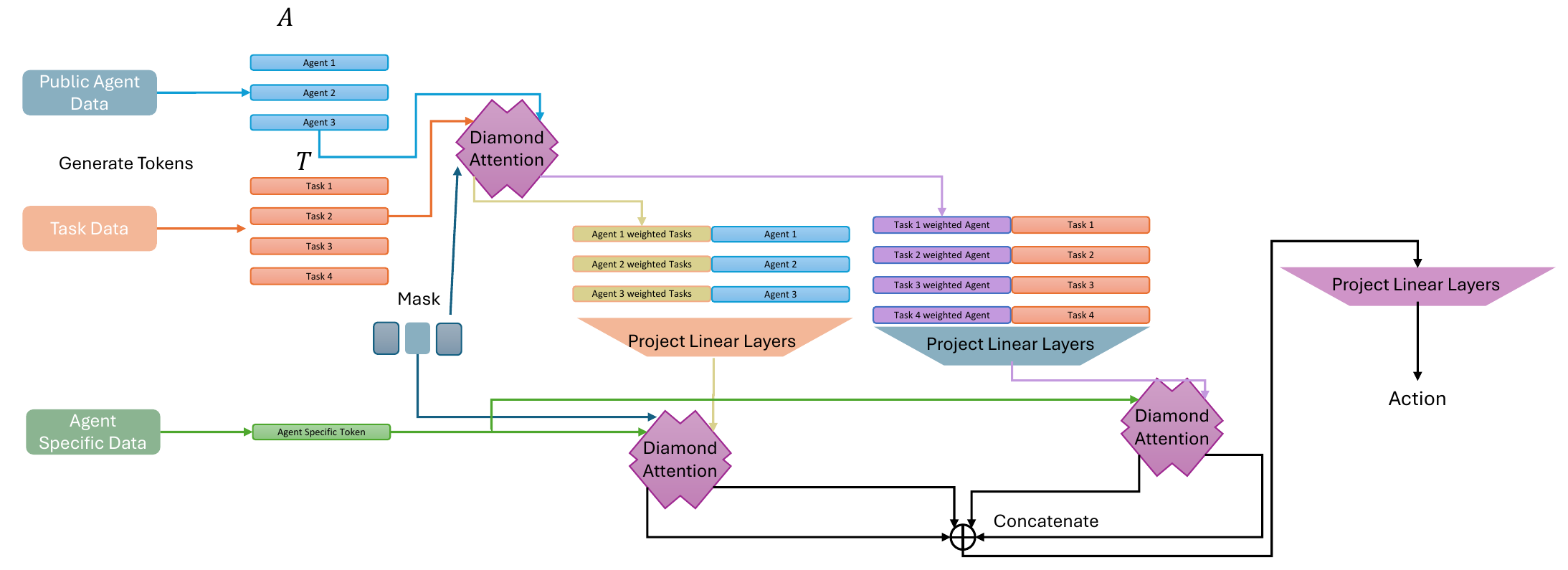}
    \caption{\small The complete architecture consists of three Diamond Attention blocks. The first computes attention between tasks and agents; agent-specific data constructs a fixed-length vector for action generation. Each \emph{linear projection} consists of multiple linear layers with intermediate activation. The depicted mask is simplified to illustrate that an agent with a lower random number is masked. Note the mask usage during Diamond Attention between agent-weighted tasks and the agent-specific 
    token.}
    \label{fig:model_architecture}
\end{figure}

\textbf{Structured mask construction.}
Each agent $i$ appends a scalar $r_i \sim \mathrm{Uniform}[0,1]$ to its observation at 
every timestep, sampled independently across agents and steps. These scalars induce a 
transient rank ordering: agent $i$ holds higher rank than agent $j$ iff $r_i > r_j$. 
Agent $i$ masks out its \emph{lower-ranked} peers from its cross-attention over agent 
embeddings, while leaving attention over task embeddings fully unmasked. Formally, for 
the $k$-th key position in agent $i$'s attention:

\begin{equation}
\label{eq:mask_construction}
\mathbf{M}_{i,k} =
\begin{cases}
  0        & \text{if } r_k \geq r_i \\
  -\infty  & \text{otherwise.}
\end{cases}
\end{equation}

Under standard parameter sharing without masking, all agents process identical network 
weights and, in symmetric configurations, produce identical action distributions. The 
structured mask breaks this symmetry at the attention level: agent $i$ with $r_i = 0.7$ 
sees a different effective attention pattern than agent $j$ with $r_j = 0.3$, because 
$i$'s mask suppresses attention to lower-ranked agents while $j$ attends to the full 
peer set. This produces differentiated latent representations, and consequently 
differentiated actions, even under fully shared parameters.

\cref{fig:masking} illustrates mask generation for three agents and three tasks.\footnote{We 
refer to subtasks as tasks, considering only cases where subtasks can be created.} The 
mask is computed independently and in parallel for each agent, requiring no sequential 
communication. The agent with the lowest $r_i$ attends to all peers and all tasks; 
agents with progressively higher ranks attend to a shrinking peer neighborhood, with the 
highest-ranked agent attending to no peers at all, conditioning its action solely on 
task structure. The mask is asymmetric across agents --- no two agents suppress the same 
subset of peers --- which is the minimal structural condition required to break the 
homogeneity that causes coordination failures in XOR. Task attention is intentionally 
left unmasked so that every agent, regardless of rank, maintains full access to the goal 
structure.

\paragraph{Task embedding construction.}
Task embeddings are constructed by projecting raw task features through a learned 
two-layer MLP with SiLU activation. In Simple Spread, each landmark's 2D relative 
position constitutes the task feature vector. In Food Collection, the relative 
positions of food items serve as task features. In SMACLite, enemy unit features (relative position, 
health, unit type) are projected into the task embedding space. Agent embeddings are 
constructed analogously from each agent's local observation (position, velocity) and 
other agents' relative positions and velocities. The architecture is scenario-agnostic: 
swapping the feature extractor suffices to apply Diamond Attention to a new domain.

\paragraph{Complete architecture.}
Given agent embeddings $A$ and subtask embeddings $T$, we calculate two sets of weighted 
embeddings $V_A$ and $V_T$ as defined in \cref{eq:diamond_attention_formula}, where 
$W_A$ and $W_T$ are learned weight matrices, $M$ is the mask (masked entries $-\infty$, 
otherwise $0$), and $\oplus$ denotes concatenation.

\begin{equation}
\label{eq:diamond_attention_formula}
    V_A = \mathrm{softmax}\!\left(\frac{TA^T}{\sqrt{d_k}} + M\right)(W_A A) \oplus A \,\,;\,\,
    V_T = \mathrm{softmax}\!\left(\frac{TA^T}{\sqrt{d_k}} + M\right)^{\!\top}(W_T T) \oplus T
\end{equation}

The complete architecture (\cref{fig:model_architecture}) comprises three Diamond 
Attention blocks. The first computes cross-attention between tasks and agents. 
Agent-specific data then constructs a fixed-length vector from which the final action is 
generated. This design accepts fixed-length embeddings per agent and task, rendering it 
independent of the total agent or task count.

%% file: sections/experiments.tex
\section{Experiments}
\label{sec:experiments}

We evaluate Diamond Attention across three settings that isolate distinct coordination
pathologies:
\textbf{The XOR Game:} Tests symmetry breaking in perfectly symmetric reward landscapes
where deterministic parameter sharing theoretically fails.
\textbf{Continuous Coordination (VMAS):} Tests zero-shot generalization to variable agent
counts without retraining in both static and non-stationary tasks.
\textbf{StarCraft Multi-Agent Challenge:} Tests coordination against an active opponent
and cross-scenario transfer with variable enemy counts.

\paragraph{Implementation details.}
We modify the PPO~\cite{schulman2017proximal} implementation from
StableBaselines3~\cite{stable-baselines3} to process observations from all agents and
produce actions using a shared policy, as described in \cref{subsec:diamond_attention}.
We employ a common critic and train using the global team reward~\cite{yu2022surprising}.
Baselines are described per experiment below.

\subsection{XOR Game: Isolating Symmetry Breaking}
\label{subsec:experiments_xor}

In an $n$-player, $k$-action XOR game, homogeneous agents sharing a deterministic policy
inevitably output identical action distributions, producing collisions that yield zero
reward~\cite{fu2022revisiting}. \cref{tab:xor_results} reports success rates. As
predicted by theory, deterministic baselines and the no-mask ablation collapse to the
random-action floor: $0.5$ in the $n=k=2$ setting and $2!/3^2 \approx 0.22$ in the $n=2, k=3$ (train) $n=3, k=3$ (eval)
setting. In contrast, our approach achieves a $1.0$ success rate across $n=k=2$ and
$n=k=3$ when trained directly on those configurations, under both greedy ($\pi^*$) and
stochastic ($\pi$) action selection. The structured random mask serves as an implicit
rank-assignment mechanism, allowing agents to assume distinct roles without explicit
communication. The $n=2, k=3 \to n=3, k=3$ generalization column tells a different
story: our method fails to generalize when trained on a strict subset of the action space
and then deployed at full team size. We discuss this generalization brittleness, and its
contrast with our generalization behavior in dynamic environments, in
\cref{subsec:learning_vs_coordination}. In case of MAT, however, it can be observed that it also fails to grasp the entire structure as greedy policy fails. However, it is able to keep some understanding of the solution as we can see that sampling from the policy gives a good success rate.

\begin{table}[h]
    \centering\small
    \begin{tabular}{|c|c|c|c|c|c|c|}
        \hline
         \textbf{Train} & \multicolumn{2}{|c|}{$n=k=2$} & \multicolumn{2}{|c|}{$n=2, k=3$} & \multicolumn{2}{|c|}{$n=k=3$} \\
         \hline
         \textbf{Eval} & \multicolumn{2}{|c|}{$n=k=2$} & \multicolumn{2}{|c|}{$n=3, k=3$} & \multicolumn{2}{|c|}{$n=k=3$}\\
         \hline
         & $\pi^*$ & $\pi$ & $\pi^*$ & $\pi$ & $\pi^*$ & $\pi$ \\
        \hline
        MAPPO & $0.0$ & 0.5 & 0.0 & 0.22 & 0.0 & 0.22\\
        \hline
        QMIX & $0.0$ & 0.5 & 0.0 & 0.22 & 0.0 & 0.22\\
        \hline
        IPPO & $0.0$ & 0.5 & 0.0 & 0.22 & 0.0 & 0.22\\
        \hline
        MASAC & $0.0$ & 0.5 & 0.0 & 0.22 & 0.0 & 0.22\\
        \hline
        pH-MARL & $0.0$ & 0.5 & 0.0 & 0.22 & 0.0 & 0.22\\
        \hline
        GSA & $0.0$ & 0.5 & 0.0 & 0.22 & 0.0 & 0.22\\
        \hline
        MAT & \textbf{1.0} & \textbf{1.0} & $0.0$ & $0.44$ & \textbf{1.0} & \textbf{1.0}\\
        \hline
        Ours & \textbf{1.0} & \textbf{1.0} & 0.0 & 0.0 & \textbf{1.0} & \textbf{1.0}\\
        \hline
        Ours w/o mask & $0.0$ & 0.5 & 0.0 & 0.22 & 0.0 & 0.22\\
        \hline
        Ours (dropout) & $0.0$ & $0.5$ & 0.0 & 0.22 & 0.0 & 0.22\\
        \hline
    \end{tabular}
    \caption{\small Average success rates on the XOR game. \textit{Train} indicates the
    training configuration, \textit{Eval} the evaluation configuration. $\pi^*$ uses
    greedy action selection, $\pi$ samples from the policy. Our method and MAT are the
    only approaches to achieve $1.0$ success on directly-trained configurations, but via
    fundamentally different mechanisms: MAT breaks symmetry through autoregressive
    decoding conditioned on a fixed agent ordering, while ours uses structured random
    masking that is agnostic to team size. The $n=2, k=3 \to n=3, k=3$ column tests
    generalization to a larger team and is discussed in
    \cref{subsec:learning_vs_coordination}.}
    \label{tab:xor_results}
    \vspace{-15pt}
\end{table}

\subsection{Continuous Coordination: Isolating Scalability}
\label{subsec:experiment_vmas}

We use the Vectorized Multi-Agent Simulator (VMAS)~\cite{bettini2022vmas} to test
zero-shot scalability of a policy trained on $N=4$ agents deployed on $N \in [2, 8]$
without retraining. Two scenarios isolate different facets:

\textbf{Simple Spread (Static Topology):} A cooperative navigation task where $N$ agents
must occupy $N$ landmarks in a 2D space $[-1, 1]^2$. Tests spatial distribution
strategies in a static environment.

\textbf{Food Collection (Dynamic Topology):} A custom foraging environment where $N_a$
agents collect $N_f$ food items that respawn at random locations on collection. Tests
adaptation to non-stationary targets. Agents receive a reward of $+20$ per collection,
with penalties for collisions.

\paragraph{Baselines and setup.}
We compare against GSA and pH-MARL~\cite{sebastian2025physics}, MAT~\cite{wen2022multi},
and MAPPO~\cite{yu2022surprising}, QMIX~\cite{rashid2020monotonic},
IPPO~\cite{de2020independent}, MASAC~\cite{lowe2017multi} via
BenchMARL~\cite{bettini2024benchmarl}, all trained on $N=4$. GSA and pH-MARL support
variable-size evaluation; the remaining baselines do not. Evaluation runs over 400
timesteps across 64 seeds using the per-agent normalized cumulative reward
$R = \frac{1}{N}\sum_{t=1}^{T} r(s,a)$~\cite{sebastian2025physics}; values are negative,
\emph{less negative is better}.

\textbf{Scalability across team sizes.}
\cref{fig:simple_spread} reports Simple Spread results. The monotonic reward decrease
with $N$ is expected, as per-agent penalties accumulate with team size. Our method
maintains near-optimal performance across the full range $N \in [2,8]$ from a single
training run, with notably low variance throughout. GSA and pH-MARL degrade
substantially at unseen team sizes, performing poorly even at small $N$ and worsening
monotonically as $N$ grows --- confirming that graph-based and physics-informed
inductive biases do not substitute for the decoupling from agent count that set-based
cross-attention provides. MAT appears only at $N{=}4$ for the architectural reasons
discussed in \cref{sec:related}.


\cref{fig:food_collection} shows Food Collection where non-stationarity makes the 
task more challenging. Our method
outperforms GSA and pH-MARL across all team sizes; both degrade sharply from $N=2$
onward, while our cross-attention mechanism remains stable across the full range $N \in
[2,8]$. Fixed-team baselines MAPPO and IPPO are competitive at the training
configuration $N=4$ but cannot be deployed elsewhere. The \textit{Ours w/o mask}
ablation performs comparably to the full model here --- a predicted consequence of our
analysis, not a failure mode. When the environment supplies sufficient asymmetry through
observation variance (stochastic food respawning differentiates agents at each step),
the protocol-space mask is redundant within distribution. The mask's contribution
becomes essential precisely when that environmental signal shifts at deployment, as the
SMAC zero-shot results in \cref{subsec:experiment_smac} confirm.
\cref{app:generalization_collection} shows performance across $N_a$ and $N_f$.

\begin{figure}[htbp]
    \centering
    \begin{subfigure}[t]{0.49\linewidth}
        \includegraphics[width=\linewidth]{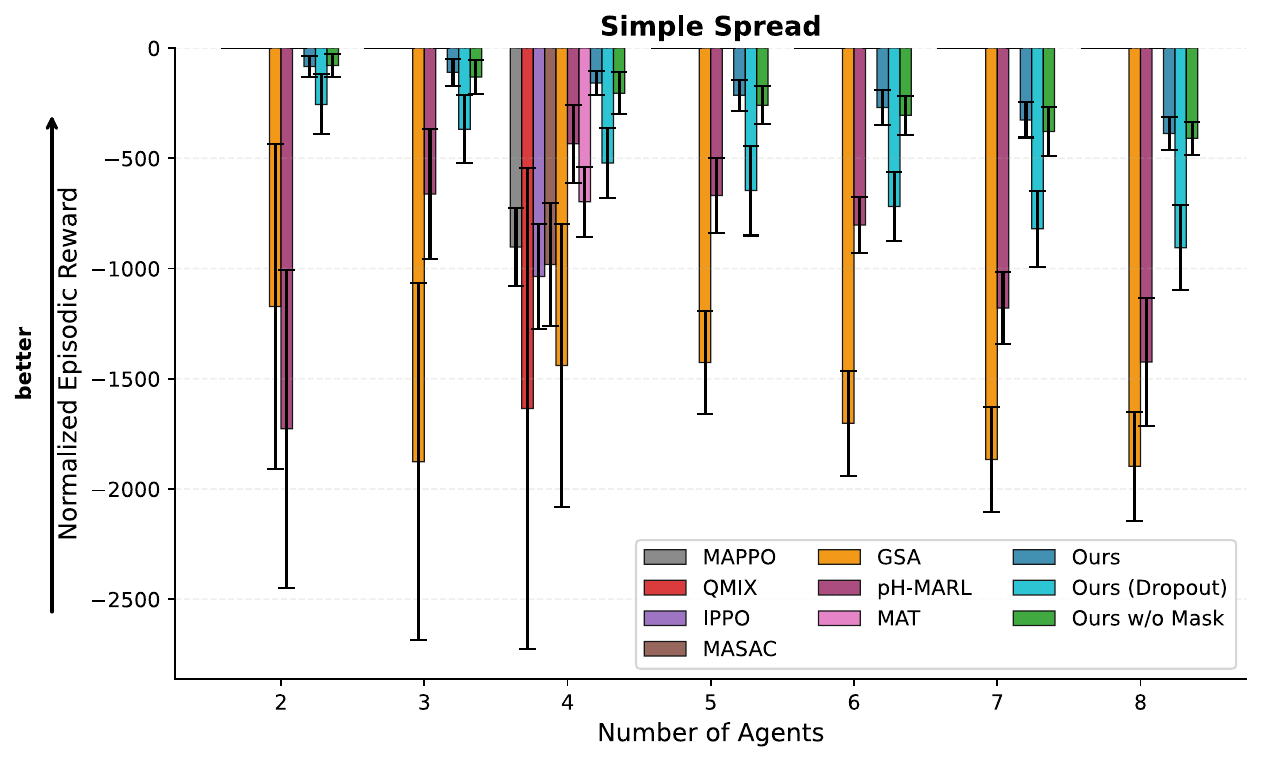}
        \caption{Simple Spread}
        \label{fig:simple_spread}
    \end{subfigure}
    ~
    \begin{subfigure}[t]{0.49\linewidth}
        \includegraphics[width=\linewidth]{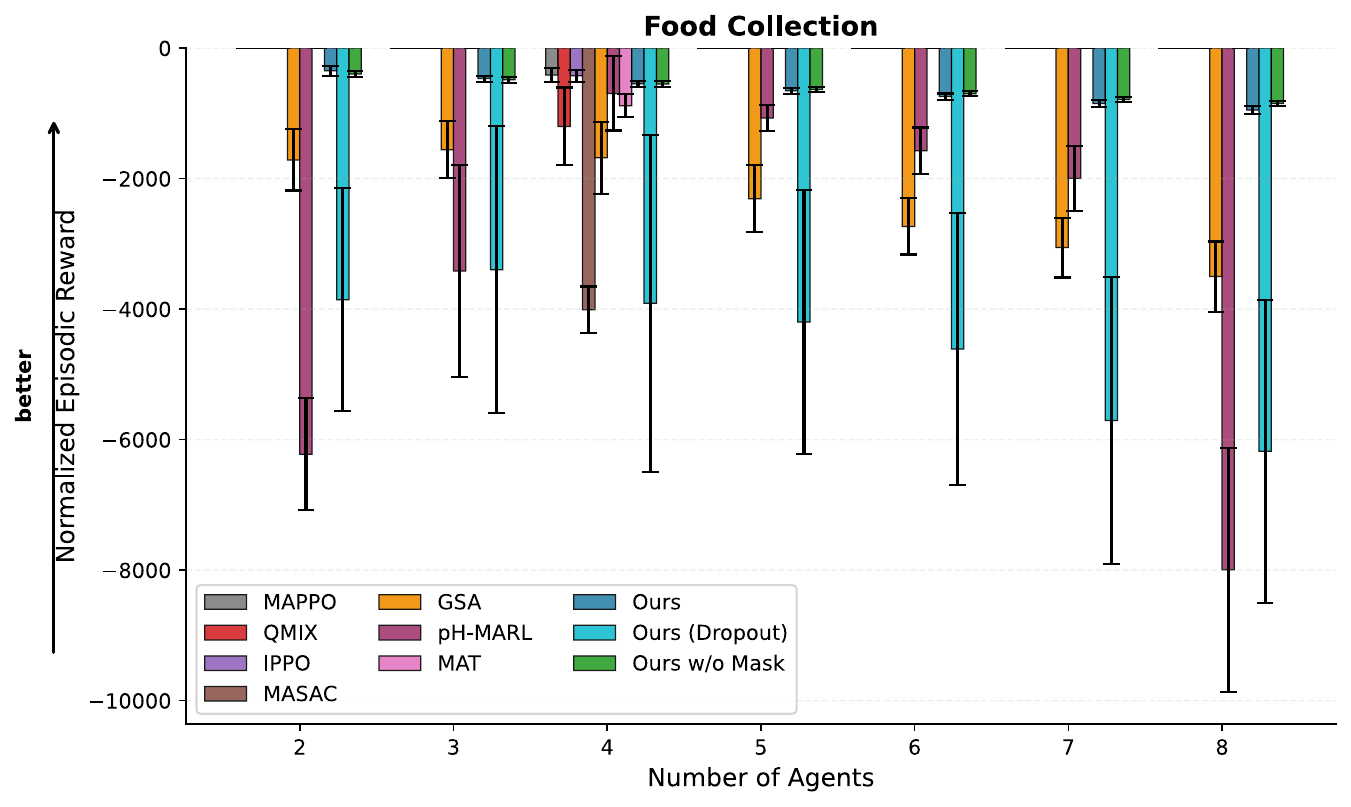}
        \caption{Food Collection}
        \label{fig:food_collection}
    \end{subfigure}
    \caption{Mean and standard deviation of normalized cumulative reward for Simple
    Spread (a) and Food Collection (b) (less negative is better; ``better'' arrow on
    left axis). Methods with fixed-size architectures (MAPPO, QMIX, IPPO, MASAC, MAT)
    appear only at the training configuration $N{=}4$, since they cannot be deployed at
    other team sizes without retraining. Methods with variable-size architectures (GSA,
    pH-MARL, Ours, Ours w/o Mask, Ours Dropout) are evaluated across $N \in [2, 8]$.
    Our method maintains robust performance across the full range despite training only
    on $N{=}4$.}
\end{figure}

\subsection{SMAC: Isolating Adversarial Robustness}
\label{subsec:experiment_smac}

We employ SMACLite~\cite{michalski2023smaclite} to test coordination in a
high-dimensional adversarial environment with rigid action spaces that typically prevent
transfer. Unlike VMAS, SMAC introduces an active opponent (the built-in AI), and the
discrete action space depends on the number of enemy units, making cross-scenario
transfer non-trivial.

\textbf{Architectural modification for transfer.}
We train on the \textbf{2s3z} scenario and evaluate zero-shot on \textbf{3s5z}. To
enable transfer across action spaces, we modify the architecture's final output layer:
the final Diamond Attention block is replaced with two linear layers sandwiching an
activation, projecting agent embeddings into opponent-dependent logits. This allows the
model to handle the changing action space (where actions depend on the number of
enemies). We do not change the size of the model. Given that the task is adversarial, our hypothesis for the degradation for the train task is because the model with structured randomness is trying to learn a more general solution.

\begin{table}[h]
    \centering
    \begin{tabular}{|c|c|c|c|c|c|c|c|}
    \hline
        & MAPPO & IPPO & QMIX & MAT & Ours & Ours & Ours \\
        &  &  &  &  &  & w/o mask & (dropout) \\
        \hline
        2s3z (Train) & $1.000$ & $1.000$ & $1.000$ & $1.000$ & $0.832$ & $1.000$ & $1.000$ \\
        \hline
        3s5z (Zero-shot) & - & - & - & - & \textbf{0.497} & $0.000$ & $0.000$ \\
        \hline
    \end{tabular}
    \caption{\small Win rates on SMAC scenarios. Standard methods achieve $100\%$ on the
    training task but cannot transfer to 3s5z due to fixed-head action-space constraints.
    Our method is capable of zero-shot transfer ($49.7\%$); removing the structured mask
    collapses transfer to $0.0\%$, and replacing the mask with dropout produces the same
    collapse.}
    \label{tab:smac_results}
    \vspace{-15pt}
\end{table}

\textbf{Results.}
\cref{tab:smac_results} reports win rates. The deterministic baselines (MAPPO, IPPO,
QMIX) achieve perfect performance on 2s3z but cannot transfer to 3s5z. Our method
achieves $49.7\%$ zero-shot transfer. Critically, the two ablations isolate \emph{what}
provides this transferability: removing the mask entirely (\textit{Ours w/o mask}) drops
transfer to $0\%$; replacing the mask with dropout does likewise, showing that
unstructured stochasticity is also insufficient. Only the \emph{structured} mask survives
the scenario shift, establishing that what enables transferable coordination is not noise
but the specific protocol-space structure of the random rank ordering.

%% file: sections/conclusion.tex
\section{Discussion, Limitations, and Impact}
\label{sec:conclusion}
The findings in this paper rest on a single distinction: coordination among homogeneous
agents requires asymmetry, and that asymmetry can be supplied either by the environment
or by the protocol itself. The two are interchangeable within distribution but not across
it. Environmental asymmetry---spatial variance, opponent-induced state shifts, observation
noise---is sufficient when training and deployment statistics align, and our no-mask
ablation on Food Collection shows this explicitly. Protocol-space asymmetry, in the form
of structured random rank ordering, is what survives when those statistics shift, as the
SMAC zero-shot transfer results demonstrate. Diamond Attention is the architectural
realization of this protocol-space asymmetry, requiring only a single broadcast round and
remaining decoupled from the specific number of agents in deployment. The contrast between
dropout and structured masking on SMAC transfer ($0\%$ vs $49.7\%$) sharpens the claim:
it is not noise that enables coordination across distribution shift, but the specific
structure of the random ordering.

\textbf{What the ablations isolate.}
\textit{Ours w/o mask} matches the full method on Food Collection
(\cref{fig:food_collection}) but fails on XOR (\cref{tab:xor_results}) and collapses to
$0\%$ transfer on SMAC (\cref{tab:smac_results}). The dropout ablation's $0\%$ confirms
the operative ingredient is the protocol-space ordering, not stochasticity itself: both
ablations confirm that neither a deterministic policy nor unstructured noise suffices when
no environmental symmetry-breaking signal is available or when that signal shifts at
deployment.

\label{subsec:learning_vs_coordination}
\textbf{Coordination capacity vs.\ learning capacity.}
Our XOR results empirically validate the theoretical bounds of~\citet{case2005learning}:
deterministic architectures cannot exceed the random-action floor in symmetric
coordination tasks, and structured randomness lifts it. A separate question is whether
the architecture can \emph{learn} to coordinate at scale. In an $n$-player, $n$-action
XOR setting the probability of non-zero reward under random play is $n!/n^n$, which
vanishes rapidly; at $n=5$ training does not converge even with reward shaping. This is a
learning bottleneck as well as architectural one, given the performance of MAT. The same architecture that solves $n=k=2$
and $n=k=3$ generalizes effectively to up to eight agents in VMAS and SMAC, where
scenario complexity---spatial proximities, continuous feedback, observable opponent
positions---provides the richer signals PPO can exploit. The coordination protocol works
in both regimes; the difference is whether PPO can find it.

\textbf{Deployment.}
Zero-shot generalization to varying agent counts implies that a single trained policy can
be deployed across fleets of fluctuating size without retraining. Because coordination
does not rely on synchronized global communication at execution time, the system tolerates
node failures and communication latency in ways that fixed-permutation autoregressive
methods cannot.

\textbf{Limitations.}
Three limitations are worth naming. First, the SMAC transfer result requires replacing the
final Diamond Attention block with a linear projection to handle variable action spaces,
trading representational capacity on the training task ($0.832$ on 2s3z against $1.000$
for fixed-head baselines) for the ability to generalize at all. Furthermore, there is a need to study the effect of structured randomness and the tradeoff of generalization and model capacity. Adaptive projection layers
that preserve coordination capacity under variable action spaces are a natural next step.
Second, the architecture's ability to learn coordination at scale is bottlenecked by
reward sparsity under PPO: convergence fails for $n > 5$ even with reward shaping, and
off-policy training or curriculum approaches may circumvent this. Third, the protocol
assumes a broadcast model in which agents share scalar values within a single timestep;
strict point-to-point deployments would require an explicit consensus layer, which would
in turn require its own coordination mechanism.

\textbf{Impact.}
This work contributes a coordination primitive whose operative state is generated
internally and stochastically per step rather than read from environmental signals.
Decentralized coordination without reliance on global synchronization or fixed agent
identities offers value in time-critical deployments where infrastructure is
unreliable---disaster response, search-and-rescue---and the internally-sampled
coordination state is harder for adversaries to predict or jam than schemes that rely on
observable environmental cues. The same properties carry real dual-use risk: the
resilience that benefits civilian applications also makes the approach attractive for
autonomous adversarial swarms or military systems operating under contested
communication. We do not see a clean mitigation here, as the underlying mechanism is
general. 

%% file: sections/appendix.tex
\section{Proof of Theorem}
\label{app:proof}
\begin{proof}
    In an $n$-homogeneous player system, each player can receive at most $(n-1)k$ bits, 
    assuming each transmits $k$ bits.

    Let $M$ be a player that plays the XOR game optimally. Assume at least one player 
    transmits $k$ bits (otherwise consider a system transmitting $k-1$ bits, which is 
    valid since players may use Routine 4 to transition to states transmitting fewer 
    bits). Construct $M'$ that runs $M$ while tracking transmitted and received bits. If 
    $M$ stops having transmitted fewer than $k$ bits, $M'$ broadcasts $0$s to reach $k$, 
    ensuring $k$ bits are transmitted and an opportunity exists to receive $k$ bits from 
    others. $M'$ is homogeneous to $M$ by definition.

    Play the game with $n$ copies of $M'$, and let $k_1, \dots, k_n$ be the $k$-length 
    bit strings transmitted. Let $\pi(k_i, k_1, \dots, k_n)$ denote player $i$'s action 
    distribution. If any two strings are identical ($k_i = k_j$), the corresponding 
    distributions are identical and a collision occurs. If all strings are distinct, 
    $\pi$ can assign probability $1$ to a unique action per player. Since homogeneous 
    agents differ only through random bits, the collision probability is minimized under 
    uniform sampling, giving:
    \[
        P(\text{No Collision}) \;\le\; \frac{\dbinom{2^{k}}{n} \cdot n!}{2^{nk}}
    \]

    Now consider $n$ homogeneous players each generating $k$ uniform random bits and 
    sharing them, then using a fixed strict ordering over $k$-bit strings to determine 
    execution order. Homogeneity ensures all players use the same ordering. A conflict 
    arises only if two strings collide, and the probability all $n$ strings are distinct 
    is exactly:
    \[
        P(\text{Different Strings}) \;=\; \frac{\dbinom{2^{k}}{n} \cdot n!}{2^{nk}}
    \]
    The two expressions are equal, establishing the equivalence.
\end{proof}

\section{Hyperparameters and Implementation details}
All training was on NVIDIA Titan X GPU.

\begin{table}[h]
\centering
\begin{tabular}{lcccc}
\hline
\textbf{Parameter} & \textbf{Food Collection} & \textbf{Simple Spread} & \textbf{XOR Game} & \textbf{SMACLite} \\
\hline

\multicolumn{5}{c}{\textit{Environment}} \\
\hline
num\_envs & 64 & 64 & 64 & 8 \\
continuous\_actions & True & True & -- & -- \\
max\_steps & 400 & 400 & -- & 108000 \\
terminated\_truncated & False & False & -- & -- \\

\hline
\multicolumn{5}{c}{\textit{PPO}} \\
\hline
batch\_size & 1024 & 1024 & 128 & 320 \\
n\_epochs & 10 & 10 & 1 & 10 \\
gamma & 0.99 & 0.99 & 0.99 & 0.99 \\
n\_steps & 160 & 160 & 2 & 160 \\
vf\_coef & 0.5 & 0.5 & 0.5 & 0.5 \\
ent\_coef & 0.001 & 0.0 & 0.0 & 0.001 \\
target\_kl & 0.25 & 0.25 & 0.25 & 0.25 \\
max\_grad\_norm & 10.0 & 10.0 & 10.0 & 10.0 \\
learning\_rate & $1 \times 10^{-4}$ & $1 \times 10^{-4}$ & $1 \times 10^{-4}$ & $1 \times 10^{-4}$ \\

\hline
\multicolumn{5}{c}{\textit{Training}} \\
\hline
total\_timesteps & 1200000 & 1200000 & 50000 & 1200000 \\

\hline
\end{tabular}
\caption{Comparison of general hyperparameters across all scripts}
\end{table}

\section{Generalization in Food Collection}
\label{app:generalization_collection}

\cref{fig:food_collection_heatmap} reports our method's performance across joint 
variations in agent count $N_a$ and food count $N_f$. Two trends emerge. For fixed 
$N_a$, normalized reward decreases as $N_f$ grows, an artifact of the reward function 
where the cumulative distance-to-food penalty scales linearly with $N_f$. For fixed 
$N_f$, $R$ improves with $N_a$ as a denser agent population covers the area more 
effectively, increasing the probability that a respawned food item appears near an 
agent. Reward variance also decreases as $N_a$ grows, indicating that the team learns a 
stable decentralized coverage strategy rather than relying on stochastic individual 
successes.

\begin{figure}[t]
    \centering
    \includegraphics[width=0.9\linewidth]{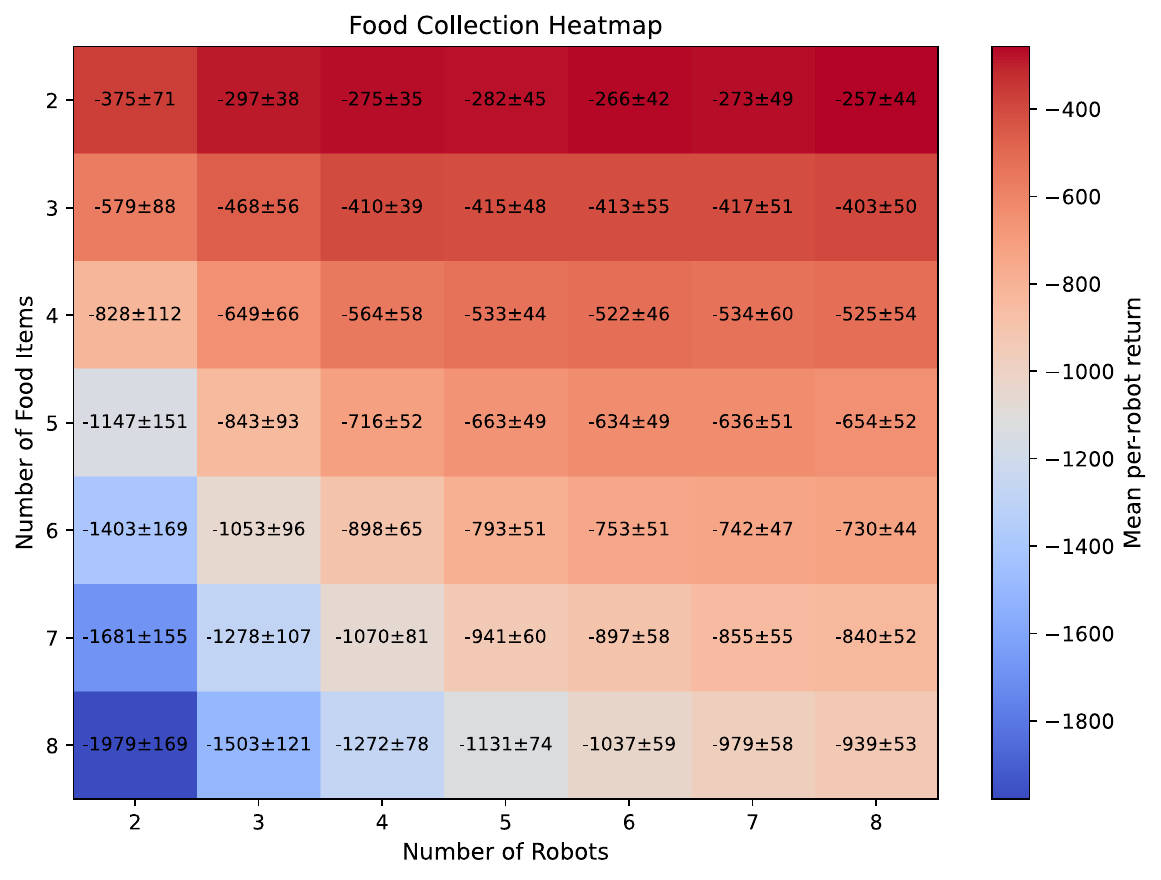}
    \caption{Mean and standard deviation of $R$ for Food Collection across 64 runs, 
    varying both agent count $N_a$ (x-axis) and food count $N_f$ (y-axis). For fixed 
    $N_a$, $R$ degrades as $N_f$ grows due to per-food distance penalties. For fixed 
    $N_f$, $R$ improves with $N_a$, and variance decreases, indicating learned 
    decentralized coverage rather than stochastic individual success.}
    \label{fig:food_collection_heatmap}
\end{figure}

\section{Attention Mechanism Visualizations}
\label{app:visualization}

These are visual snapshots of the per-agent cross-attention weights in evaluation mode to offer deeper insight into the dynamic masking process. 

Figures~\ref{fig:spread_snapshots} and~\ref{fig:food_snapshots} illustrate the environment state alongside the cross-attention heatmaps for a 4-agent team in the Simple Spread and Food Collection scenarios, respectively, at different timesteps. The dynamic nature of the randomized structured mask is clearly visible: at any given timestep, the attention matrices are structurally asymmetric across the team. Agents assigned a higher transient rank by the masking protocol exhibit highly concentrated attention, effectively acting as temporary "leaders" for specific subtasks, while lower-ranked agents exhibit broader attention across their peers. As the random numbers are resampled, this fluid hierarchy shifts continuously throughout the episode, successfully breaking coordination symmetry without requiring explicit communication.

\begin{figure}[t]
    \centering
    \includegraphics[width=0.7\textwidth]{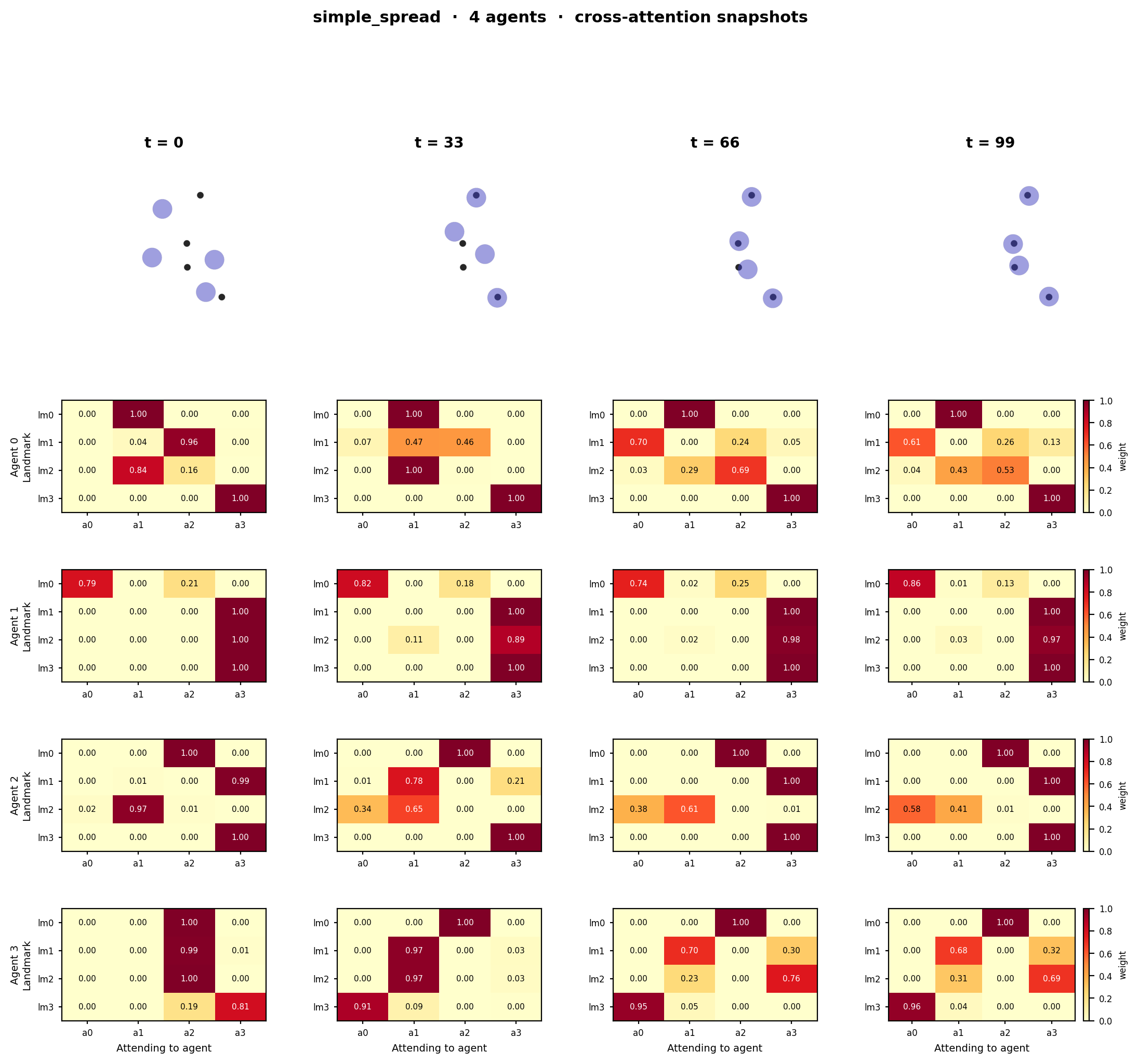}
    \caption{Cross-attention snapshots for 4 agents in the Simple Spread scenario. The top row displays the spatial distribution of agents and landmarks at various timesteps. The subsequent rows display the corresponding attention weights for Agents 0 through 3. The heatmaps confirm the application of the structured mask, revealing the transient, asymmetric attention hierarchies that enable dynamic coordination.}
    \label{fig:spread_snapshots}
\end{figure}

\begin{figure}[t]
    \centering
    \includegraphics[width=0.7\textwidth]{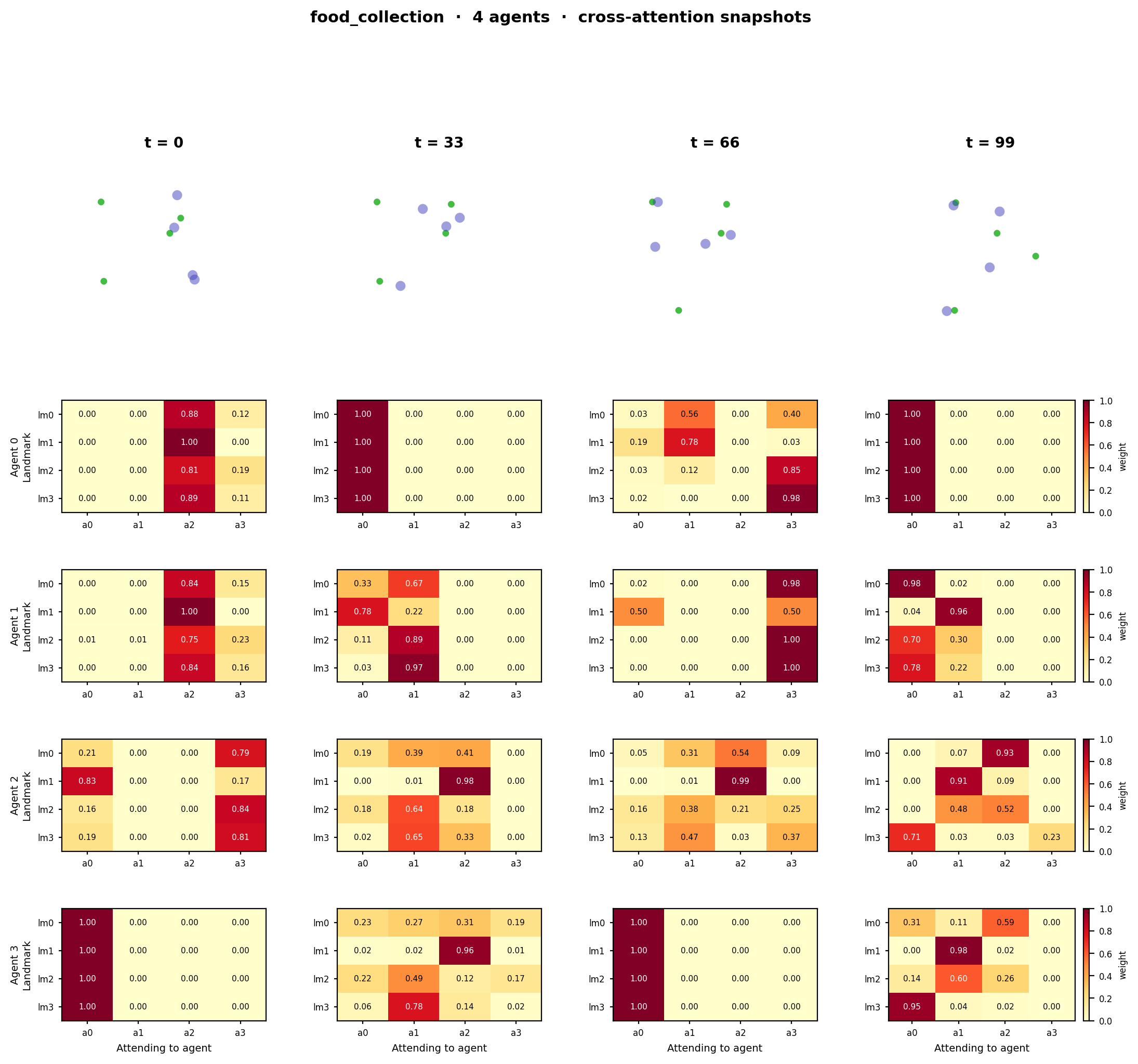}
    \caption{Cross-attention snapshots for 4 agents in the Food Collection scenario. Similar to the Simple Spread task, the top row displays the spatial distribution of agents (blue) and food items (green), with the rows below showing the shifting, asymmetric attention weights that drive decentralized coordination over the course of the non-stationary episode.}
    \label{fig:food_snapshots}
\end{figure}